\documentclass[11pt,letterpaper]{article}

\usepackage[utf8]{inputenc} 
\usepackage[T1]{fontenc}    
\usepackage{hyperref}       
\usepackage{url}            
\usepackage{booktabs}       
\usepackage{amsfonts}       
\usepackage{nicefrac}       
\usepackage{microtype}      
\usepackage{chngpage}       
\usepackage{float}
\usepackage{graphicx}
\usepackage{subcaption}
\usepackage{natbib}
\usepackage{fullpage}

\author{John Hainline\quad Brendan Juba\quad Hai S.~Le\\
Washington University in St.~Louis\\
{\tt \{john.hainline,bjuba,hsle\}@wustl.edu}\and 
David Woodruff\\
Carnegie Mellon University\\
{\tt dwoodruf@cs.cmu.edu}}

\usepackage{amsmath,amssymb,paralist,algorithm2e}

\newtheorem{theorem}{Theorem}
\newtheorem{lemma}[theorem]{Lemma}
\newtheorem{definition}[theorem]{Definition}

\newtheorem{corollary}[theorem]{Corollary}

\def\FullBox{\hbox{\vrule width 8pt height 8pt depth 0pt}}

\newcommand{\qed}{\;\;\;\FullBox}

\newenvironment{proof}{\noindent{\bf Proof:~~}}{\qed}
\newenvironment{proof-of}[1]{\noindent{\bf Proof of {#1}:~~}}{\qed}

\newcommand{\bbE}{\mathbb{E}}
\newcommand{\bbR}{\mathbb{R}}

\newcommand{\defeq}{\overset{\mathrm{def}}{=}}

\newcommand{\vspan}{\mathrm{span}}

\title{Conditional Sparse $\ell_p$-norm Regression With Optimal Probability\thanks{B.~Juba and H.~S.~Le were supported by an AFOSR Young Investigator Award and NSF Award CCF-1718380. D. Woodruff would like to thank IBM Almaden, where part of this work was done, as well as support by the National Science Foundation under Grant No. CCF-1815840.}}

\begin{document}
\maketitle

\begin{abstract}
We consider the following {\em conditional linear regression} problem:
the task is to identify both (i) a $k$-DNF\footnote{%
$k$-Disjunctive Normal Form ($k$-DNF):  an OR of ANDs of at most $k$ {\em literals}, where a literal is 
either a Boolean attribute or the negation of a Boolean attribute.}  condition $c$ and (ii) a linear rule
$f$ such that the probability of $c$ is (approximately) at least some 
given bound $\mu$, and $f$ minimizes the $\ell_p$ loss of predicting the
target $z$ in the distribution of examples {\em conditioned on} $c$. Thus, the task is to
identify a portion of the distribution on which a linear rule can provide a
good fit. Algorithms for this task are useful in cases where simple, learnable rules only accurately model portions of the 
distribution. The prior state-of-the-art for such algorithms 
could only guarantee finding a condition of probability $\Omega(\mu/n^k)$ when a 
condition of probability $\mu$ exists, and achieved an $O(n^k)$-approximation to
the target loss, where $n$ is the number of Boolean attributes. Here, we give efficient algorithms for
solving this task with a condition $c$ that nearly matches the probability of the
ideal condition, while also improving the approximation to the target loss.
We also give an algorithm for finding a $k$-DNF {\em reference class} for
prediction at a given query point, that obtains a sparse regression fit that 
has loss within $O(n^k)$ of optimal among all sparse regression parameters and 
sufficiently large $k$-DNF reference classes containing the query point.
\end{abstract}

\section{Introduction}
In areas such as advertising, it is common to break a population into 
{\em segments}, on account of a belief that the population as a whole is
heterogeneous, and thus better modeled as separate subpopulations. A natural
question is how we can identify such subpopulations. A related
question arises in personalized medicine. Namely, we need existing cases to
apply data-driven methods, but how can we identify cases that should be grouped 
together in order to obtain more accurate models? 

In this work, we consider a problem of this sort recently formalized by Juba~\citeyear
{juba17}: we are given a data set where each example has both a real-valued vector together with a Boolean-valued vector associated with it. Let $\vec{Y}$ denote the real-valued vector for some example, and let $\vec{X}$ denote the vector of Boolean attributes. We are also given some target prediction $Z$ of interest, and we wish to identify a $k$-DNF $c$ such that $c(\vec{X})$ specifies a subpopulation on which we can find a linear predictor for $Z$, $\langle \vec{a},\vec{Y}\rangle$, that achieves small loss. 
In particular, it may be that neither $\vec{a}$, nor any other linear predictor 
achieves small loss over the entire data distribution, and we simply wish to 
find the subset of the distribution, determined by our $k$-DNF $c$ evaluated on the Boolean attributes of our data, on which accurate regression is possible. We refer to this subset of the distribution as a "segment".

Juba proposed two algorithms for different families of loss functions, one for 
the sup norm (that only applies to $O(1)$-sparse regression) and one for the 
usual $\ell_p$ norms. But, the algorithm for the $\ell_p$ norm has a major 
weakness, namely that it can only promise to identify a condition $c'(\vec{X})$
that has a probability polynomially smaller than optimal. This is especially problematic if the segment had relatively small 
probability to begin with: it may be that there is no such condition with enough
examples in the data for adequate generalization. By contrast, the algorithm for
the sup norm recovers a condition $c'$ with probability essentially the same as 
the optimal condition, but of course the sup norm is an undesirable norm to use 
for regression, as it is maximally sensitive to outliers. In general,
$\ell_p$ regression is more tolerant to outliers as $p$ decreases, and
the sup norm is roughly ``$p=\infty$.''

In this work, in Theorem~\ref{mainthm}, we show how to give an algorithm for $O(1)$-sparse $\ell_p$ 
regression with a two-fold improvement over the previous
$\ell_p$ regression algorithm (mentioned above) when the coefficient vector is sparse:
\begin{compactenum}
\item The new algorithm recovers a condition $c'$ with probability
essentially matching the optimal condition.
\item We also obtain a smaller blow-up of the loss relative to the optimal 
regression parameters $\vec{a}$ and optimal condition $c$, reducing the degree of this polynomial factor.
\end{compactenum}
More concretely, Juba's previous algorithm identifies $k$-DNF
conditions, as does ours. This algorithm only identified a 
condition of probability $\Omega(\mu/ n^k)$ when a condition with 
probability $\mu$ exists, and only achieved loss
bounded by $O(\epsilon n^k)$ when it is possible to achieve loss $\epsilon$.
We improve the probability of the condition to $(1-\eta)\mu$ for any desired
$\eta>0$, and also reduce the bound on the loss to $\tilde{O}(\epsilon n^{k/2})$,
or further to $\tilde{O}(\epsilon g\log\log(n^k))$ if the condition has only $g$ ``terms'' (ANDs of literals), in Corollary~\ref{slpr-cor}.
This latter algorithm furthermore features a smaller sample complexity, that
only depends logarithmically on the number $n$ of Boolean attributes.
We include a synthetic data experiment demonstrating that the latter algorithm
can successfully recover more terms of a planted solution than Juba's sup norm algorithm.

We also give an algorithm for the closely related {\em reference class 
regression} problem in Theorem~\ref{rcreg-thm}. In this problem, we are given a query point $\vec{x}^*$ 
(along with values for the real attributes $\vec{y}^*$), and we wish to estimate 
the corresponding $z^*$. In order to compute this estimate for a single point, 
we find some $k$-DNF $c'$ such that $c'(\vec{x}^*)=1$ so $\vec{x}^*$ is in the 
support of the distribution conditioned on $c'(\vec{X})$, and such that there 
are (sparse) regression parameters $\vec{a}'$ for $c'$ that (nearly) match the 
optimal loss for sparse regression under any condition $c$ with 
$c(\vec{x}^*)=1$. In this case, the condition $c'$ is a {\em reference class} 
containing $\vec{x}^*$, obtaining the tightest possible fit among $k$-DNF 
classes containing $\vec{x}^*$. We think of this as describing a collection of 
``similar cases'' for use in estimating the target value $z^*$ from the 
attributes $\vec{y}^*$. For this problem, we guarantee that we find a condition 
$c'$ and sparse regression parameters $\vec{a}'$ that achieve loss within 
$O(n^k)$ of optimal.

Our algorithms are based on sparsifiers for linear systems, for example as
obtained for the $\ell_2$ norm by Batson et al.~\citeyear{bss12}, and as obtained for 
non-Euclidean $\ell_p$ norms by Cohen and Peng~\citeyear{cp15} based on Lewis weights~\citep
{lewis78}. The strategy is similar to Juba's sup norm algorithm: there, by
enumerating the vertices of the polytopes obtained on various subsets of
the data, he is able to obtain a list of all possible candidates for the
optimal estimates of the sup-norm regression coefficients. Here, we can 
similarly obtain such a list of estimates for the $\ell_p$ norms by enumerating
the possible sparsified linear systems. Sparsifiers are often used to 
accelerate algorithms, e.g., to run in time in terms of the number of
nonzero entries rather than the size of the overall input, though here we use that the size of a sparsified representation is small in a rather different way. Namely, the small representation allows for a feasible enumeration of possible candidates for the regression coefficients. 

We note that this means that we also obtain an algorithm
for $O(1)$-sparse $\ell_p$-norm regression in the {\em list-learning}
model~\citep{bbv08,csv17}. Although Charikar et al.~in
particular are able to solve a large family of problems in the 
list-learning model, they observe that their technique can only obtain 
trivial estimates for regression problems. We show that once we have such a 
list, algorithms for the {\em conditional distribution search} 
problem~\citep{juba16,zmj17,jlm18} can be generically used to extract a 
description of a good event $c$ for the given candidate
linear rule. Similar algorithms can also then be used to find a good
reference class. 

\subsection{Relationship to other work}
The conditional linear regression problem is most closely related to the problem of prediction with a ``reject'' option. In this model, a prediction algorithm has the option to abstain from making a prediction. In particular, El-Yaniv and Weiner~\citeyear{eyw12} considered linear regression in such a model. However, like most of the work in this area, their approach is based on scoring the ``confidence'' of a prediction function, and only making a prediction when the confidence is sufficiently high. This does not necessarily yield a nice description of the region in which the predictor will actually make predictions. On the other hand, Cortes et al.~\citeyear{cdsm16} consider an alternative variant of learning with rejection that does obtain such nice descriptions; but, in their model, they assume that abstention comes at a fixed, known cost, and they simply optimize the overall loss when their space of possible prediction values includes this fixed-cost ``abstain'' option. By contrast, we can explicitly consider various rates of coverage or loss. 

Our work is also related to the field
of {\em robust statistics}~\citep{huber81,rl87}: in this area, one seeks to
mitigate the effect of outliers. That is, roughly, we suppose that a small
fraction of the data has been corrupted, and we wish for our models and
inferences to be relatively unaffected by this corrupted data. The difference
between robust statistics and the conditional regression problem is that in
conditional regression we are willing to ignore {\em most} of the data. We only
wish to find a small fraction, described by some rule, on which we can obtain a
good regression fit.

Another work that similarly focuses on learning for small subsets of the data is
the the work by Charikar et al.~\citeyear{csv17} in the list-learning model of Balcan et al.~\citeyear{bbv08}. In that work, one seeks to find a
setting of some parameters that (nearly) minimizes a given loss function on an
unknown small subset of the data. Since it is generally impossible to
produce a single set of parameters to solve this task -- many different
small subsets could have different choices of optimal parameters -- the objective
is to produce a small list of parameters that contains a near-optimal
solution for the unknown subset somewhere in it. As we mentioned above, our
approach actually gives a list-learning algorithm for $O(1)$-sparse $\ell_p$
regression; but, we show furthermore how to identify a conditional distribution
such that some regression fit in the list obtains a good fit, thus also solving
the conditional linear regression task. Indeed, in general, we find that
algorithms for solving the list-learning task for regression yield algorithms for
conditional linear regression. 
The actual technique of Charikar et al.~does not solve our problem since 
their approximation guarantee essentially 
always admits the zero vector as a valid approximation. 
Thus, they 
also do not obtain an algorithm for list-learning (sparse) regression either.

Another task similar to list-learning of linear regression is the problem of 
finding dense linear relationships, as solved by RANSAC~\citep{fb81}. But, these
techniques only work in constant dimension. By contrast, although we are seeking
sparse regression rules, this is a sparse fit in high dimension. As with 
list-learning, in contrast to the conditional regression problem, such algorithms do not provide a description of the points fitting the dense linear relationship. 

Finally, yet another task similar to list-learning for regression is fitting
linear mixed models~\citep{mccs01,jiang07}. In this approach, one seeks to
explain all (or almost all) of the data as a mixture of several linear rules.
The guarantee here is incomparable to ours: in contrast to list-learning, the
linear mixed model simply needs a list of linear rules that accounts for nearly
all of the data; it does not need to find a list that accounts for all possible
sufficiently large subsets of the data. So, there is no guarantee that any of
the mixture components represent an approximation to the regression parameters
corresponding to an event of interest. On the other hand, in linear mixed 
models, one does not need to give any description at all of which points should
lie in which of the mixture components. In applications, one usually assigns
points to the linear rule that gives it the smallest residual, but
this may be less useful for predicting the values for new points.

\section{Preliminaries}

We now formally define the problems we consider in this work, and recall the
relevant background.

\subsection{The conditional regression and distribution search tasks}\label{problems}

Formally, we focus on the following task:

\begin{definition}[Conditional $\ell_p$-norm Regression]
{\em Conditional $\ell_p$-norm linear regression} is the following task. Suppose
we are given access to i.i.d.~examples drawn from a joint distribution over
$(\vec{X},\vec{Y},Z)\in\{0,1\}^n\times\{\vec{y}\in\bbR^d:\|y\|_2\leq b\}\times
[-b,b]$ such that for some $k$-DNF $c^*$ and some coefficient vector $\vec{a}^*
\in\bbR^d$ with $\|\vec{a}^*\|_2\leq b$, $\bbE[|\langle \vec{a}^*,\vec{Y}\rangle
-Z|^p|c^*(\vec{X})]\leq\epsilon$ and $\Pr[c^*(\vec{X})]\geq (1+\eta)\mu$ for 
some given $b>0$, $\epsilon>0$, $\eta>0$, and $\mu>0$. Then we wish to find 
$\hat{c}$ and $\hat{\vec{a}}$
such that $\bbE[|\langle\hat{\vec{a}},\vec{Y}\rangle-Z|^p|\hat{c}(\vec{X})]^{1/p}\leq
\alpha\epsilon$ and $\Pr[\hat{c}(\vec{X})]\geq (1-\eta)\mu$ for an
{\em approximation factor} function $\alpha$.

If $\vec{a}^*$ has at most $s$ nonzero components, and we require $\hat{\vec{a}}$ likewise
has at most $s$ nonzero components, then this is the {\em conditional $s$-sparse
$\ell_p$-norm regression} task.
\end{definition}

As stressed by Juba~\citeyear{juba17}, the restriction to $k$-DNF conditions is not 
arbitrary. If we could identify conditions that capture arbitrary conjunctions,
even for one-dimensional regression, this would yield PAC-learning algorithms
for general DNFs. In addition to this being an unexpected breakthrough, recent 
work by Daniely and Shalev-Shwartz~\citeyear{dss16} shows that such algorithms would imply new algorithms for
random $k$-SAT, falsifying a slight strengthening of Feige's hypothesis~\citep
{feige02}. We thus regard it as unlikely that any algorithm can hope to find
conditions of this kind. Of the classes of Boolean functions that do not contain
arbitrary conjunctions, $k$-DNFs are the most natural large class, and
hence are the focus for this model.

The difficulty of the problem lies in the fact that initially we are given neither the condition nor the linear predictor. Naturally, if we are told what the relevant subset is, we can just use standard methods for linear regression to obtain the linear predictor; conversely, if we are given the linear predictor, then we can use algorithms for the {\em conditional distribution search (learning abduction) task} (introduced by Juba~\citeyear{juba16}, recalled below) to identify a condition on which the linear predictor has small error. Our final algorithm actually uses this connection, by considering a list of candidates for the linear predictors to use for labeling the data, and choosing the linear predictor that yields a condition that selects a large subset. 
In particular, we will use algorithms for the following {\em weighted}
variant of the {\em conditional distribution search} task:

\begin{definition}[Conditional Distribution Search]
{\em Weighted conditional distribution search} is the following problem.
Suppose we are given access to examples drawn i.i.d.~from a distribution 
over $(\vec{X},W)\in\{0,1\}^n\times [0,b]$ such that there exists a
$k$-DNF condition $c^*$ with $\Pr[c^*(\vec{X})]\geq (1+\eta)\mu$
and $\bbE[W|c^*(\vec{X})]\leq\epsilon$ for some given parameters $\eta,\mu,b,
\epsilon>0$. Then, find a $k$-DNF $\hat{c}$ such that 
$\Pr[\hat{c}(\vec{X})]\geq(1-\eta)\mu$ and 
$\bbE[W|\hat{c}(\vec{X})]<\alpha\cdot\epsilon$ for some {\em approximation
factor} function $\alpha$ (or INFEASIBLE if no such $c^*$ exists).
\end{definition}

This task is closely related to {\em agnostic} conditional distribution search, 
which is the special case where the weights only take values $0$ or $1$.
The current state of the art for agnostic conditional distribution search is an
algorithm given by Zhang et al.~\citeyear{zmj17}, achieving an 
$\tilde{O}(\sqrt{n^k})$-approximation to the optimal error. That 
work built on an earlier algorithm 
due to Peleg~\citeyear{peleg07}. Peleg's work already showed how to extend his original
algorithm to a weighted variant of the problem, and we observe that an
analogous modification of the algorithm used by Zhang et al.~will obtain an
algorithm for the more general weighted conditional distribution search problem 
we are considering here:

\begin{theorem}[Peleg~\citeyear{peleg07}, Zhang et al.~\citeyear{zmj17}]\label{wcds-alg-thm}
There is a polynomial-time algorithm for weighted conditional distribution search
achieving an 
$\tilde{O}(n^{k/2}(\log b+\log 1/\eta+\log\log 1/\delta))$-approximation 
with probability $1-\delta$ given
$
m = \Theta\left(\frac{b^3}{\mu\epsilon\eta^2}(n^k+\log\frac{1}{\delta})\right)\text{ examples.}
$
\end{theorem}

We note that it is possible to obtain much stronger guarantees when we are seeking 
a small formula. Juba et al.~\citeyear{jlm18} present an algorithm that, when there
is a $k$-DNF with $g$ terms achieving error $\epsilon$, uses only $m=
\tilde{O}(\frac{gk}{\mu\epsilon\eta^2}\log\frac{n}{\delta})$ examples and obtains a 
$k$-DNF with probability at least $(1-\eta)\mu$ and error $\tilde{O}(\epsilon g
\log m)$. Although this is stated for the unweighted case (i.e., $\epsilon$ is a 
probability), it is easy to verify that since our loss is nonnegative and bounded, by
rescaling the losses to lie in $[0,1]$, we can obtain an analogous guarantee for the
weighted case:
\begin{theorem}[Juba et al.~\citeyear{jlm18}]\label{scds-alg-thm}
There is a polynomial-time algorithm for weighted conditional distribution search
when the condition has $g$ terms using
$m=\tilde{O}(\frac{bgk}{\mu\epsilon\eta^2}\log\frac{n}{\delta})$ examples
achieving an $\tilde{O}(g\log m)$-approximation using a $k$-DNF with 
$\tilde{O}(g\log m)$ terms with probability $1-\delta$.
\end{theorem}

\subsection{Reference class regression}
Using a similar approach, we will also solve a related problem, selecting a best
$k$-DNF ``reference class'' for regression. In this task, we are not merely 
seeking some $k$-DNF event of probability $\mu$ on which the conditional loss is
small. Rather, we are given some specific observed Boolean attribute values
$\vec{x}^*$, and we wish to find a $k$-DNF condition $c$ {\em that is satisfied
by $\vec{x}^*$} solving the previous task. That is, $c$ should have probability
at least $\mu$ and our sparse regression fit has small conditional loss, 
conditioned on $c$. Naturally, the motivation here is that we have some specific
point $(\vec{x}^*,\vec{y}^*)$ for which we are seeking to predict $z^*$, and so
we are looking for a ``reference class'' $c$ such that we can get the tightest
possible regression estimate of $z^*$ from $\vec{y}^*$; to do so, we need to take
$\mu$ large enough that we have enough data to get a high-confidence estimate,
and we need $\vec{x}^*$ to lie in the support of the conditional distribution for
which we are computing this estimate. 

\begin{definition}[Reference Class $\ell_p$-Regression]
{\em Reference class $\ell_p$-norm regression} is the following task. We are given a 
{\em query point} $\vec{x}^*\in\{0,1\}^n$, target density $\mu\in (0,1)$, ideal loss bound $\epsilon_0>0$
approximation parameter $\eta>0$, confidence parameter $\delta>0$, and access to i.i.d.
examples drawn from a joint distribution over $(\vec{X},\vec{Y},Z)\in\{0,1\}^n\times
\{\vec{y}\in\bbR^d:\|\vec{y}\|_2\leq b\}\times[-b,b]$. We wish to find $\hat{\vec{a}}
\in\bbR^d$ with $\|\hat{\vec{a}}\|_2\leq b$ and a {\em reference class} $k$-DNF 
$\hat{c}$ such that with probability $1-\delta$, 
\begin{inparaenum}
\item[(i)] $\hat{c}(\vec{x}^*)=1$, 
\item[(ii)] $\Pr[\hat{c}(\vec{X})]\geq(1-\eta)\mu$, and
\item[(iii)] for a fixed {\em approximation factor} $\alpha>1$,
$
\bbE[|\langle \hat{\vec{a}},\vec{Y}\rangle-Z|^p|\hat{c}(\vec{X})]^{1/p}\leq \alpha
\max\{\epsilon^*,\epsilon_0\}
$
where $\epsilon^*$ is the optimal $\ell_p$ loss $\bbE[|\langle\vec{a}^*,\vec{Y}\rangle-
Z|^p|c^*(\vec{X})]^{1/p}$ over $\vec{a}^*\in\bbR^d$ of $\|\vec{a}^*\|_2\leq b$ and $k$-DNFs 
$c^*$ such that $c^*(\vec{x}^*)=1$ and $\Pr[c^*(\vec{X})]\geq\mu$.
\end{inparaenum}
If we also require both $\hat{\vec{a}}$ and $\vec{a}^*$ to have at most $s$ nonzero
components, then this is the {\em reference class $s$-sparse $\ell_p$-norm regression} task.
\end{definition}

The selection and use of such reference classes for estimation goes back to work 
by Reichenbach~\citeyear{reichenbach49}. Various refinements of this approach were proposed by 
Kyburg~\citeyear{kyburg74} and Pollock~\citeyear{pollock90}, e.g., to choose the estimate provided by 
the highest-accuracy reference class that is consistent with the most specific 
reference class containing the point of interest $\vec{x}^*$. Our approach is not
compatible with these proposals, as they essentially disallow the use of the kind
of {\em disjunctive} classes that are our exclusive focus. Along the lines we 
noted earlier, it is unlikely that there exist efficient algorithms for selecting
reference classes that capture arbitrary conjunctions, so $k$-DNFs are 
essentially the most expressive class for which we can hope to solve this task. 
Bacchus et al.~\citeyear{bghk96} give a nice discussion of other unintended shortcomings of 
disallowing disjunctions. A concrete example discussed by Bacchus et al.~is
the genetic disease {\em Tay-Sachs}. Tay-Sachs only occurs in two very specific, distinct 
populations: Eastern European Jews and French Canadians. 
Thus, a study of 
Tay-Sachs should consider a reference class at least partially defined by a 
disjunction over membership in these two populations.

\subsection{$\ell_p$ sparsification}
Our approach is based on techniques for extracting low-dimensional sketches of
small subspaces in high dimensions. The usual $\ell_2$ norm uses
much simpler underlying techniques, and we describe it first. The extension to
$\ell_p$ norms for $p\neq 2$ is obtained via {\em Lewis weights}~\citep{lewis78}.

\subsubsection{Euclidean sparsifiers}
The kind of sketches we need originate in the work of Batson et al.~\citeyear{bss12}.
Specifically, it will be convenient to start from the following variant due to
Boutsidis et al.~\citeyear{bdmi14}:
\begin{lemma}[BSS weights~\citep{bdmi14}]\label{bss-weights-lem}
Let $[u]\in\bbR^{d\times t}$ ($t<d$) be the matrix with rows $\vec{u}_1,\ldots,
\vec{u}_d$ such that $\sum_{i=1}^d u_iu_i^\top = I_t$. Then given an integer
$r\in (t,d]$, there exist $s_1,\ldots,s_d\geq 0$ such that at most $r$ of the
$s_i$ are nonzero and for the $d\times r$ matrix $[s]$ with $i$th column
$\sqrt{s_i}\vec{e}_i$,
\[
\lambda_t([u]^\top [s][s]^\top [u])\geq (1-\sqrt{t/r})^2 \text{ and }
\lambda_1([u]^\top [s][s]^\top [u])\leq (1+\sqrt{t/r})^2
\]
where $\lambda_j$ denotes the $j$th largest eigenvalue.
\end{lemma}
In particular, taking $r=t/\gamma^2$ for some $\gamma\in (0,1)$, we obtain that for
the $[s]$ guaranteed to exist by Lemma~\ref{bss-weights-lem},
$
\|[s]^\top[u]\vec{v}\|_2^2 = ([s]^\top [u]\vec{v})^\top ([s]^\top [u]\vec{v})
$
for any $\vec{v}$, and hence by Lemma~\ref{bss-weights-lem},
$
(1-\gamma)\|\vec{v}\|_2\leq \|[s][u]\vec{v}\|_2\leq (1+\gamma)\|\vec{v}\|_2.
$
Furthermore, we can bound the magnitude of the entries of $[s]$ for orthonormal $[u]$ as follows:
\begin{lemma}\label{bss-weights-bound}
Suppose the rows of $[u]$ are orthonormal. Then the matrix $[s]$ obtained by
Lemma~\ref{bss-weights-lem} has entries of magnitude
at most $(1+\sqrt{t/r})\sqrt{d}$.
\end{lemma}
\begin{proof}
Observe that since the each $i$th row of $[u]$ has unit norm, it must have an
entry $u_{i,j^*}$ that is at least $1/\sqrt{d}$ in magnitude. By the above
argument,
\[
\|[s]^\top[u]\vec{e}_{j^*}\|_2^2\leq (1+\sqrt{t/r})^2\|\vec{e}_{j^*}\|_2^2=(1+\sqrt{t/r}
)^2
\]
where notice in particular, the $i$th row of $[s]^\top$ contributes at least
$(\sqrt{s_i}u_{i,j^*})^2\geq s_i/d$ to the norm. Thus, $s_i\leq (1+\sqrt{t/r})^2d
$.
\end{proof}

\subsubsection{Sparsifiers for non-Euclidean norms}
It is possible to obtain an analogue of the BSS weights for $p\neq 2$ using
techniques based on {\em Lewis weights}
\citep{lewis78}. Lewis weights are a general way to reduce problems involving
$\ell_p$ norms to analogous $\ell_2$ computations. Cohen and Peng~\citeyear{cp15} applied this
to sparsification to obtain the following family of sparsifiers:

\begin{theorem}[$\ell_p$ weights~\citep{cp15}]\label{lp-weights-thm}
Given a $d\times t$ matrix $[u]$ there exists a set of $r(p,t,\gamma)$ weights
$s_1,\ldots, s_r$ such that for the $d\times r$ matrix $[s]$ which has as its
$i$th column $s_i\vec{e}_i$,
\[
(1-\gamma)\|[u]\vec{v}\|_p\leq \|[s]^\top [u]\vec{v}\|_p\leq
(1+\gamma)\|[u]\vec{v}\|_p
\]
where $r(p,t,\gamma)$ is asymptotically bounded as in Table~\ref{dim-table}.
\end{theorem}

\begin{table}[h]
\caption{Dimension required for $(1\pm\gamma)$-approximate $\ell_p$
sparsification of $t$-dimensional subspaces. $p=2$ uses BSS weights.}\label{dim-table}
\begin{center}
\begin{tabular}{c|l}
$p$ & Required dimension $r$\\
\hline
$p=1$ & $\frac{t\log t}{\gamma^2}$\\
$1<p<2$ & $\frac{1}{\gamma^2}t\log(t/\gamma)\log^2\log(t/\gamma)$\\
$p=2$ & $t/\gamma^2$\\
$p>2$ & $\frac{\log 1/\gamma}{\gamma^5}t^{p/2}\log t$
\end{tabular}
\end{center}
\end{table}

Cohen and Peng also show how to construct the sparsifiers for a given matrix
efficiently, but we won't be able to make use of this, since we will be searching for the sparsifier for an unknown subset of the rows.

We furthermore obtain an analogue of Lemma~\ref{bss-weights-bound} for the
$\ell_p$ weights, using essentially the same argument:

\begin{lemma}\label{lp-weights-bound}
Suppose the rows of $[u]$ are orthonormal. Then the matrix $[s]$ obtained by
Theorem~\ref{lp-weights-thm} has entries of magnitude at most $(1+\gamma)
\sqrt{d}$.
\end{lemma}

\section{The weighting algorithms for conditional and reference class regression}

The results from sparsification of linear systems tell us that we can estimate the loss on the subset of the data selected by the unknown condition by computing the loss on an appropriate ``sketch,'' a weighted average of the losses on a small subset of the data. The dimensions in Table~\ref{dim-table} give the size of these sparsified systems, i.e., the number of examples from the unknown subset we need to use to estimate the loss over the whole subset. The key point is that since the predictors are sparse, the dimension is small; and, since we are willing to accept a constant-factor approximation to the loss, a small number of points suffice. Therefore, it is feasible to enumerate these small tuples of points to obtain a list of candidate sets of points for use in the sketch. We also need to enumerate the weights for these points, but since we have also argued that the weights are bounded and we are (again) willing to tolerate a constant-factor approximation to the overall expression, there is also a small list of possible approximate weights for the points. The collection of points, together with the weights, gives a candidate for what might be an appropriate sketch for the empirical loss on the unknown subset. We can use each such candidate approximation for the loss to recover a candidate for the linear predictor. Thus, we obtain a list of candidate linear predictors that we can use to label our data as described above.
More precisely, our algorithm is as shown in Algorithm~\ref{walg}.
In the following, let $\Pi_{d_1,\ldots,d_s}$ denote the projection to
coordinates $d_1,d_2,\ldots,d_s$. 
\begin{algorithm} 
\DontPrintSemicolon
\SetKwInOut{Input}{input}\SetKwInOut{Output}{output}
\SetKwInOut{Subroutines}{subroutines}
\Input{Examples $(\vec{x}^{(1)},\vec{y}^{(1)},z^{(1)}),\ldots,(\vec{x}^{(m)}, \vec{y}^{(m)},z^{(m)})$, target loss bound $\epsilon$ and fraction $\mu$.}
\Output{A $k$-DNF over $x_1,\ldots,x_n$ and linear predictor over $y_1,\ldots, y_d$, or INFEASIBLE if none exist.}
\Subroutines{{\sf WtCond} takes as inputs examples $(\vec{x}^{(1)}, \ldots,\vec{x}^{(m)})$, nonnegative weights $(w^{(1)},\ldots,w^{(m)})$, and a bound $\mu$, and returns a $k$-DNF $\hat{c}$ over $x_1,\ldots,x_n$ solving the weighted conditional distribution search task.}

\Begin{
Let $m_0=\lceil\frac{1}{\mu}(\frac{b^{2p}}{\gamma^{2p}\epsilon^{2p}}(2pb+\sqrt{2\ln(12/\delta)})^2+\ln\frac{3}{\delta})\rceil$, $r$ is as given in Table~\ref{dim-table}.\\
\ForAll{$(d_1,\ldots,d_s)\in {[d]\choose s}$, 
$(q_1,\ldots,q_r)\in \{-\lceil \frac{1}{\gamma}(\ln r-\frac{1}{p}\ln\gamma)\rceil,
\ldots, 0,\ldots, \lceil \ln(s+1)/2\gamma \rceil\}$ and 
$(j_1,\ldots,j_r)\in {[m]\choose r}$}
{
  Let $\vec{a}$ be a solution to the following convex optimization problem:
  minimize $\sum_{\ell=1}^r((1+\gamma)^{q_\ell}
   (\langle \vec{a},\Pi_{d_1,\ldots,d_s}\vec{y}^{(j_\ell)}\rangle-z^{(j_\ell)})^p$ subject to $\|\vec{a}\|_p\leq b$.

  Put $c\gets$ the output of {\sf WtCond} on $(\vec{x}^{(1)},
\ldots,\vec{x}^{(m)})$ with the weights $w^{(i)}=|\langle\vec{a},\Pi_{d_1,\ldots,d_s}\vec{y}^{(i)}\rangle-z^{(i)}|^p$ and bound $\mu$.\\
  \lIf{{\sf WtCond} did not return INFEASIBLE and 
$\hat{\bbE}[(\langle\vec{a},\Pi_{d_1,\ldots,d_s}\vec{Y}\rangle-Z)^p|
c(\vec{X})]^{1/p}\leq\alpha\epsilon$}
  {
    \Return{$\vec{a}$ and $c$.}
  }
}
\Return{INFEASIBLE.}
}
\caption{Weighted Sparse Regression}\label{walg}
\end{algorithm}

\begin{theorem}[Conditional sparse $\ell_p$-norm regression]\label{mainthm}
For any constant $s$ and $\gamma>0$, $r$ as given in Table~\ref{dim-table} for 
$t=(s+1)$, and $m=\Theta\left(\frac{((1+\gamma)b)^3}{\mu\epsilon\eta^2}(n^k+s\log d+r\log\frac{m_0\log(\gamma^{1/p}s/r)}{\gamma}+\log\frac{1}{\delta})+m_0\right)$
examples,
Algorithm~\ref{walg} runs in polynomial time and solves the conditional 
$s$-sparse $\ell_p$ regression task with $\alpha=\tilde{O}((1+\gamma)\sqrt{n^k}
(\log b+\log\frac{1}{\eta}+\log\log\frac{1}{\delta})\epsilon)$.
\end{theorem}
In our proof of this theorem, we will find it convenient to use the Rademacher generalization bounds for linear predictors (note that $x\mapsto |x|^p$ is $pb^{p-1}$-Lipschitz on $[-b,b]$):
\begin{theorem}[Bartlett and Mendelson~\citeyear{bm02}, Kakade et al.~\citeyear{kst09}]\label{rcgen}
For $b>0$, $p\geq 1$, random variables $(\vec{Y},Z)$ distributed over
$\{\vec{y}\in\bbR^d:\|y\|_2\leq b\}\times [b,b]$, and any $\delta\in
(0,1)$, let $L_p(\vec{a})$ denote $\bbE[|\langle\vec{a},\vec{Y}
\rangle-Z|^p]$, and for an an i.i.d.~sample of size $m$ let
$\hat{L}_p(\vec{a})$ be the empirical loss $\frac{1}{m}\sum_{j=1}^m
|\langle \vec{a},\vec{y}^{(j)}\rangle-z^{(j)}|^p.$
We then have that with probability $1-\delta$ for all
$\vec{a}$ with $\|a\|_2\leq b$,
\[
|L_p(\vec{a})- \hat{L}_p(\vec{a})|\leq\frac{2pb^{p+1}}{\sqrt{m}}+b^p\sqrt{\frac{2\ln(4/\delta)}{m}}.
\]
\end{theorem}

Note that although this bound is stated in terms of the $\ell_2$ norm of the attribute and parameter vectors $\vec{y}$ and $\vec{a}$, we can obtain a bound in terms of the dimension $s$ of the sparse rule if we are given a bound $B$ on the magnitude of the entries: $b\leq \sqrt{s}B$.

\begin{proof-of}{Theorem~\ref{mainthm}}
Given that we are directly checking the empirical $\ell_p$ loss before returning
$\vec{a}$ and $c$, for the quoted number of examples $m$ it is immediate by a union bound over the iterations that
any $\vec{a}$ and $c$ we return are satisfactory with probability $1-\delta$.
All that needs to be shown is that the algorithm will find a pair that passes
this final check.

By Theorem~\ref{rcgen}, we note that it suffices to have
$\frac{b^{2p}}{\gamma^{2p}\epsilon^{2p}}(2pb+\sqrt{2\ln(12/\delta)})^2$
examples from the distribution conditioned on the unknown $k$-DNF event $c^*$ to obtain that the $\ell_p$ loss of each candidate for $\vec{a}$ is estimated to within an additive $\gamma\epsilon$ with probability $1-\delta/3$.
By Hoeffding's inequality, therefore when we draw $m_0$ examples, there is
a sufficiently large subset satisfying $c^*$ with probability $1-\delta/3$.

We let $[u]$ be an orthonormal basis for
$\vspan\{(\Pi_{d_1,\ldots,d_s}\vec{y}^{(j)},z^{(j)}):c^*(\vec{x}^{(j)})=1, j\leq m_0\}$
and invoke Lemma~\ref{bss-weights-lem} for $\ell_2$ or
Theorem~\ref{lp-weights-thm} for $p\neq 2$. In either case,
there is some set of weights $s_1,\ldots,s_{r_0}$ for a subset of $r_0$
coordinates $j_1,\ldots,j_{r_0}$ such that for any $\vec{v}$ in the column span of $[u]$,
$[s]^\top[u]\vec{v}$ has $\ell_p$ norm that is a $1\pm\gamma$-approximation to
the $\ell_p$ norm of $[u]\vec{v}$. In particular, for any $\vec{a}$, observing
$\vec{v}=[y]\vec{a}-\vec{z}$ is in the column span of $[u]$ by construction, we
obtain
\[
(1-\gamma)\|[y]\vec{a}-\vec{z}\|_p\leq \|[s]^\top([y]\vec{a}-\vec{z})\|_p\leq
(1+\gamma)\|[y]\vec{a}-\vec{z}\|_p.
\]
Now, we observe that we can discard weights (and dimensions) from $s_1,\ldots,
s_{r_0}$ of magnitude smaller than $\gamma/r_0^{1/p}$, since for any unit
vector $\vec{v}$, the contribution of such entries to $\|[s]^\top [u]
\vec{v}\|_p^p$ (recalling there are at most $r_0$ nonzero entries) is at most
$\gamma^p$. So we may assume the $r\leq r_0$ remaining weights all have
magnitude at least $\gamma/r^{1/p}$. Furthermore, if we round each weight to
the nearest power of $(1+\gamma)$, this only changes $\|[s]^\top [u]
\vec{v}|_p^p$ by an additional $(1\pm\gamma)$ factor. Finally, we note that
since $(\Pi_{d_1,\ldots,d_s}\vec{y}^{(j)},z^{(j)})$ has dimension $s+1$,
Lemmas~\ref{bss-weights-bound} and \ref{lp-weights-bound} guarantee that the
magnitude is also at most $(1+\gamma)\sqrt{s+1}$. Thus it indeed suffices
to find the powers $(q_1,\ldots,q_r)$ for our $r$ examples $j_1,\ldots,j_r$
such that $(1+\gamma)^{q_\ell}$ is within $(1+\gamma)$ of $s_\ell$, and the
resulting set of weights will approximate the $\ell_p$-norm of every $\vec{v}$
in the column span to within a $1+3\gamma$-factor.

Now, when the loop in Algorithm~\ref{walg} considers (i) the dimensions $d^*_1,
\ldots, d^*_s$  contained in the optimal
$s$-sparse regression rule $\vec{a}^*$
(ii) the set of examples $j^*_1,\ldots,j^*_r$ used for the sparse approximation for these coordinates and
(iii) the appropriate weights $(1+\gamma)^{q^*_1},\ldots,(1+\gamma)^{q^*_r}$,
the algorithm will obtain a vector $\vec{a}$ that achieves a
$(1+3\gamma)$-approximation to the empirical $\ell_p$-loss of $\vec{a}^*$ on the
same $s$ coordinates.

It then follows from Theorem~\ref{wcds-alg-thm} that with probability at least
$1-\delta/3$ over the data, {\sf WtCond} will in turn return to us
a $k$-DNF $c$ with probability $(1-\eta)\mu$ that selects a subset of the data
on which $\vec{a}$ achieves an $\alpha\epsilon=
\tilde{O}((1+\gamma)\sqrt{n^k}(\log b+\log 1/\eta+\log\log 1/\delta)\epsilon)$
approximation to the empirical $\ell_p$ loss of $\vec{a}^*$ on $c^*$. This
choice of $\vec{a}$ and $c$ passes the final check and is thus sufficient.
\end{proof-of}

By simply plugging in the algorithm from Theorem~\ref{scds-alg-thm} for {\sf WtCond}, 
we can obtain the following improvement when the desired $k$-DNF condition is small:
\begin{corollary}\label{slpr-cor}
For any constant $s$ and $\gamma>0$, $r$ as given in Table~\ref{dim-table} for 
$t=(s+1)$, and
$m = \Theta\left(\frac{(1+\gamma)bgk}{\mu\epsilon\eta^2}\left(\log\frac{n}{\delta}+s\log d+r\log\frac{m_0\log(\gamma^{1/p}s/r)}{\gamma}\right)
+m_0\right)$
examples,
if there is a $g$-term $k$-DNF solution to the conditional $s$-sparse 
$\ell_p$-regression task, then the modified Algorithm~\ref{walg} runs in polynomial time 
and solves the task with a $\tilde{O}(g\log m)$-term $k$-DNF with $\alpha=
\tilde{O}((1+\gamma)\epsilon g \log m)$.
\end{corollary}

Note that this guarantee is particularly strong in the case where the $k$-DNF would be 
small enough to be reasonably interpretable by a human user.

The extension to reference class $\ell_p$-norm regression proceeds by replacing the 
weighted condition search algorithm with a variant of the tolerant elimination algorithm 
from Juba~\citeyear{juba16}, given in Algorithm~\ref{refclassalg}.
\begin{algorithm}
\DontPrintSemicolon
\SetKwInOut{Input}{input}\SetKwInOut{Output}{output}
\Input{Examples $(\vec{x}^{(1)},w^{(1)}),\ldots,(\vec{x}^{(m)},w^{(m)})$, query point $\vec{x}^*$, minimum fraction $\mu_0$, minimum loss target $\epsilon_0$, approximation parameter $\eta$.}
\Output{A $k$-DNF over $x_1,\ldots,x_n$.}
\Begin{
  Initialize $\mu\gets 1$, $\hat{c}\gets\bot$, $\hat{\epsilon}\gets\max_jw^{(j)}$\\
  \While{$\mu\geq\mu_0$}{
    Initialize $\epsilon\gets\hat{\epsilon}$\\
    \While{$\epsilon\geq\epsilon_0/(1+\eta)$}{
      Initialize $c$ to be the empty disjunction\\
      \ForAll{Terms $T$ of at most $k$ literals}{
        \lIf{$\sum_{j:T(\vec{x}^{(j)})=1}w^{(j)}\leq \epsilon\mu m$}{
          Add $T$ to $c$.
        }
      }
      Put $\epsilon\gets\epsilon/(1+\eta)$
    }
    \lIf{$c(\vec{x}^*)=1$, $\sum_{j=1}^mc(\vec{x}^{(j)})\geq\mu m$, and $\epsilon<\hat{\epsilon}$}{
      Put $\hat{c}\gets c$, $\hat{\epsilon}\gets\epsilon$
    }
    Put $\mu\gets\mu/(1+\eta)$
  }
  \Return{$\hat{c}$}
}
\caption{Reference Class Search}\label{refclassalg}
\end{algorithm}

%

\begin{lemma}\label{classalg-lem}
If $m\geq\Omega(\frac{b^3}{\eta^2(\epsilon_0+\epsilon^*)\mu_0}(k\log n+\log\frac{1}{\eta\delta}+\log\log\frac{1}{\mu_0}+\log\log\frac{b}{\epsilon_0}))$ where $W\in [0,b]$, then Algorithm~\ref{refclassalg} returns a $k$-DNF $\hat{c}$ such that with probability $1-\delta$,
\begin{inparaenum}
\item $\hat{c}(\vec{x}^*)=1$
\item $\Pr[\hat{c}(\vec{X})]\geq\mu_0/(1+\eta)$
\item $\bbE[W|\hat{c}(\vec{X})]\leq O((1+\eta)^4n^k(\epsilon_0+\epsilon^*))$ where $\epsilon^*$ is the minimum $\bbE[W|c^*(\vec{X})]$ over $k$-DNF $c^*$ such that $c^*(\vec{x}^*)=1$ and $\Pr[c^*(\vec{X})]\geq (1+\eta)\mu_0$.
\end{inparaenum}
\end{lemma}
\begin{proof}
For convenience, let $N\leq (1+\log_{1+\eta}b/\epsilon_0)(1+\log_{1+\eta}1/\mu_0)$ denote the total number of iterations.
Consider first what happens when the loop considers the largest $\mu\leq \Pr[c^*(\vec{X})]/(1+\eta)$ and the smallest $\epsilon$ that is at least $(1+\eta)^2\epsilon^*$. On this iteration, for each term $T$ of $c^*$, we observe that $\bbE[W\cdot T(\vec{X})]\leq\epsilon^*\Pr[c^*(\vec{X})]$---indeed,
\[
\bbE[W\cdot T(\vec{X})]\leq \bbE[W\cdot c^*(\vec{X})]\leq \epsilon^*\Pr[c^*(\vec{X})].
\]
So, since $W$ is bounded by $b$, by a Chernoff bound, $\frac{1}{m}\sum_{j:T(x^{(j)})=1}w^{(j)}\leq (1+\eta)\epsilon^*\Pr[c^*(\vec{X})]$ with probability $1-\frac{\delta}{2{n\choose \leq k}N}$. Since this is in turn at most $\epsilon\mu$, $T$ will be included in $c$ on this iteration. But similarly, for $T$ not in $c$ with $\bbE[W\cdot T(\vec{X})]>(1+\eta)\mu\epsilon$, the Chernoff bound also yields that $\sum_{j:T(x^{(j)})=1}w^{(j)}\geq \epsilon\mu m$ with probability $1-\delta/2{n\choose \leq k}N$. By a union bound over all $T\in c^*$ and $T$ not in $c^*$ with such large error, we see that with probability at least $1-\delta/2N$, all of the terms of $c^*$ are included in $c$ and only terms with $\bbE[W\cdot T(\vec{X})]\leq(1+\eta)\mu\epsilon$ are included in $c$. So,
\begin{align*}
\bbE[W\cdot c(\vec{X})]&\leq\sum_{T\text{ in }c}\bbE[W\cdot T(\vec{X})]\\
&\leq O((1+\eta)n^k\mu\epsilon).
\end{align*}
Furthermore, by yet another application of a Chernoff bound, $c^*$ is true of at least $\mu m$ examples with probability at least $1-\delta/2N$. Thus, with probability $1-\delta/N$,
after this iteration $\hat{c}$ is set to some $k$-DNF and $\hat{\epsilon}\leq (1+\eta)^2\max\{\epsilon^*,\epsilon_0\}$.

Now, furthermore, on every iteration, we see more generally that with probability $1-\delta/N$, only terms with $\bbE[W\cdot T(\vec{X})]\leq(1+\eta)\mu\epsilon$ are included in $c$, and $\hat{c}$ is only updated if $c(\vec{x}^*)=1$ and $\Pr[c(\vec{X})]\geq \mu/(1+\eta)$, where $\mu\geq\mu_0$. Thus, for the $\hat{c}$ we return, since $\bbE[W\cdot c(\vec{X})]\leq O((1+\eta)n^k\mu\epsilon)$ and $\bbE[W|c(\vec{X})]=\bbE[W\cdot c(\vec{X})]/\Pr[c(\vec{X})]$, $\bbE[W|\hat{c}(\vec{X})]\leq O((1+\eta)^2n^k\hat{\epsilon})$. Thus, with probability $1-\delta$ overall, since we found above that $\hat{\epsilon}\leq (1+\eta)^2\max\{\epsilon^*,\epsilon_0\}$, we return a $k$-DNF $\hat{c}$ as claimed.
\end{proof}

Now, as noted above, our algorithm for reference class regression is obtained essentially by substituting Algorithm~\ref{refclassalg} for the subroutine {\sf WtCond} in Algorithm~\ref{walg}; the analysis, similarly, substitutes the guarantee of Lemma~\ref{classalg-lem} for Theorem~\ref{wcds-alg-thm}. In summary, we find:

\begin{theorem}[Reference class regression]\label{rcreg-thm}
For any constant $s$ and $\gamma>0$, $r$ as given in Table~\ref{dim-table} for 
$t=(s+1)$, and $m \geq m_0+\Omega\left(\frac{(1+\gamma)^3b^3}{\eta^2(\epsilon_0+\epsilon^*)\mu_0}\left(
r\log\frac{m_0\log(\gamma^{1/p}s/r)}{\gamma}+\log\left(\frac{n^kd^s}{\eta\delta}\log\frac{1}{\mu_0}\log\frac{(1+\gamma)b}{\epsilon_0}\right)
\right)\right)$
examples,
our modified algorithm runs in polynomial time and solves the reference class 
$s$-sparse $\ell_p$ regression task with $\alpha=O((1+\gamma)(1+\eta)^4n^k)$.
\end{theorem}

\section{Experimental evaluation}

To evaluate our algorithm's performance in practice, we two kinds of experiments: one using synthetic data with a planted solution, and another using some standard benchmark data sets from the LIBSVM repository~\citep{chang11}.

\paragraph{Synthetic data.} To generate the synthetic data sets, we first chose a random 2-DNF over Boolean attributes by sampling terms uniformly at random. We also fixed two out of $d$ coordinates at random and sampled parameters $\vec{a}$ from a mean $0$, variance $\sigma^2$ Gaussian for our target regression fit. 
For the actual data, we then sampled Boolean attributes $\vec{X}$ that, with 25\% 
probability are a uniformly random satisfying assignment to the DNF and with 75\% 
probability are a uniformly random falsifying assignment. We sampled $\vec{Y}$ from a 
standard Gaussian in $\bbR^{d}$. Finally, for the examples where $\vec{X}$ was a 
satisfying assignment, we let $Z$ be given by the linear rule $\langle 
\vec{a},\vec{Y}\rangle+ \nu$ where $\nu$ is a mean $0$, variance $\sigma^2$ Gaussian; 
otherwise, if $\vec{X}$ was a falsifying assignment, $Z$ was set to $\nu^*$ where $\nu^*$ is a standard  mean $0$, variance $1$ Gaussian. (Parameters in Table \ref{table:dataset-comparison}.) Thus, conditioned on the planted formula, there is a small-error regression fit $\vec{a}$---the expected error of $\vec{a}$ is $|\nu|$. Off of this 
planted formula, the expected error for the optimal prediction $z=0$ is $|\nu^*|=1$. 

\begin{table}[t]
  \begin{adjustwidth}{-.5in}{-.5in}
    \caption{Synthetic Data Set Properties}
    \label{table:dataset-comparison}
    \centering
    \begin{tabular}{ccccc}
      \toprule
      Size & Real Attributes ($d$) & Boolean Attributes ($\vec{X}$) & $k$-DNF Terms  &  Variance $\sigma^2$ \\
      \midrule
      1000    & 6   & 10         & 4                & 0.01   \\
      5000    & 10  & 50         & 16              & 0.1    \\
      \bottomrule
    \end{tabular}
  \end{adjustwidth}
\end{table}

For these simple data sets, we made several modifications to simplify and accelerate the algorithm. First, we simply fixed all of the weights (of the form $(1+\gamma)^{q_i}$) to $1$, i.e., $q_i=0$ for all $i$, since the leverage scores for Gaussian data (and indeed, most natural data sets) are usually small. Second, we set $m_0=200$ for the 1000-example data set and $m_0=500$ for the 5000-example data set. Third, we used the version of the algorithm for small DNFs considered in Corollary~\ref{slpr-cor}, that uses algorithm of Juba et al.~\citeyear{jlm18} (recalled in Theorem~\ref{scds-alg-thm}) for {\sf WtCond}. 

With these modifications, we ran Algorithm~\ref{walg} for $\ell_2$ regression with $s=2$, $\gamma=1$, and $\mu=0.22$ on the 1000-example data set.
The algorithm was able to identify the planted $2$-DNF together with the two components used in the regression rule in this case, and thus also we approximately identify the linear predictor. The mean squared error of the approximate predictor returned by our algorithm was $0.0149$ on the planted $2$-DNF (c.f.~the ideal linear fit has variance-$0.01$ Gaussian noise added to it). However, we note that once the planted $2$-DNF has been identified, a good regression fit can be found by any number of methods. We also ran Algorithm~\ref{walg} for $\ell_2$ regression with $s=2$, $\gamma=1$, and $\mu=0.2465$ on the 5000-example data set. The algorithm was able to identify a subset of $11$ out of the $16$ planted $k$-DNF terms, picking up about $98\%$ of the planted $k$-DNF data. The mean squared error of the approximate predictor was $0.4221$.

As a baseline, we also used the algorithm from Juba~\citeyear{juba17} for conditional sparse sup-norm regression on our synthetic data. In order to obtain reasonable performance, we needed to make some modifications: we modified the algorithm along the lines of Algorithm~\ref{walg}, to only search over tuples of a random $m_0=200$ examples for the 1000-example data set and $m_0=500$ for the 5000-example data set in producing its candidate regression parameters. We also chose to take the parameters that achieved the smallest residuals under the sup norm, among those that obtained a condition satisfied by the desired fraction of the data. Again, we used $s=2$, and for the norm bound we used a threshold of $\epsilon = 0.24$ for the 1000-example data set and $1.1$ for the 5000-example data set. (We considered a range of possible thresholds; when $\epsilon\approx 1.5$ on the $\sigma^2=0.1$ data, the sup norm algorithm starts adding terms outside the planted DNF, significantly increasing its error.) For the 1000-example data set, this algorithm was able to identify the planted $2$-DNF. However, for the 5000-example data set, it only found 10 out 16 terms of the planted $2$-DNF, picking up about $96\%$ of the planted $k$-DNF data. 

\paragraph{Real-world benchmarks.}
We also compared Algorithm~\ref{walg} against \textit{selective regressors}~\citep{eyw12} on some of the LIBSVM \citep{chang11} regression data sets that were used to evaluate that work. We split each of those datasets into a training set (1/3 of the data) and test set (2/3 of the data). We generated Boolean attributes by choosing different binary splits (median, quartile) on the numerical features, excluding the target attribute. We then run our algorithm \ref{walg} on the training set to obtain a $2$-DNF. Next, we filtered both the training and test data satisfying the $2$-DNF, using the first one to train a new regression fit on the selected subset and the other as a holdout to estimate the error of this resulting regression fit. We compared this to the test error for selective regression. Risk-Coverage (RC) curves of the results are shown in Figure \ref{fig:RCCurves}. We outperform the baseline on the Boston housing dataset, and we generally achieved lower error than the baseline for lower coverage (less than 0.5) and somewhat higher error for higher coverage for the other three data sets (Body fat, Space, and Cpusmall). Note that in addition, we obtained a $2$-DNF that describes the subpopulation on all four, which is the main advantage of our method. Since selective regression first tries to fit the entire data set, and then chooses a subset where that fixed predictor does well, it is to be expected that it may miss a small subset where a different predictor can do better, but that its freedom to abstain on a somewhat arbitrary subset may give it an advantage at high coverage.

\begin{figure}[ht]
\begin{center}
\begin{subfigure}{.495\linewidth}
  \centering
  \includegraphics[width=1\linewidth]{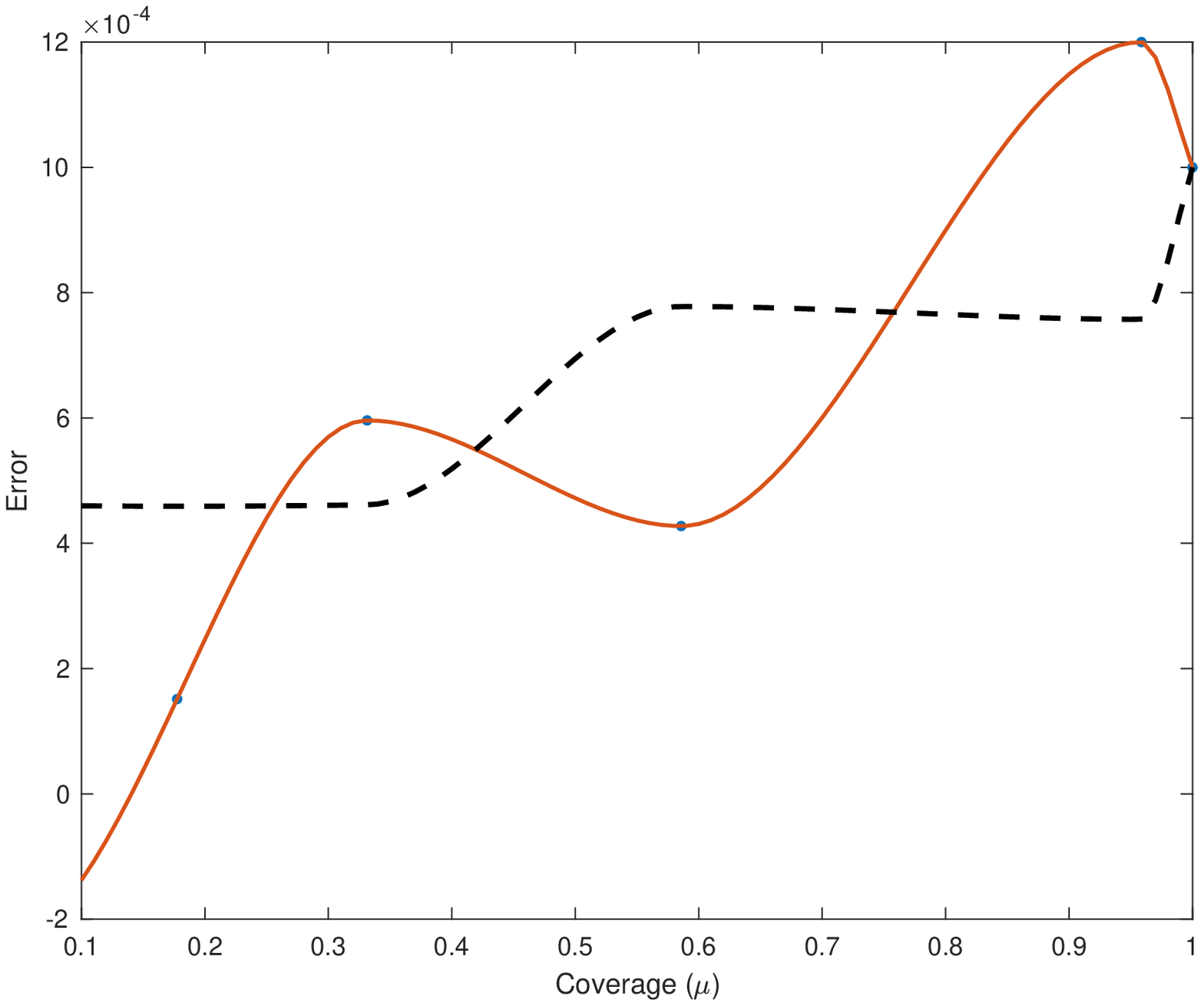}
  \caption{Body fat}
  \label{fig:bodyfat}
\end{subfigure}%
\begin{subfigure}{.495\linewidth}
  \centering
  \includegraphics[width=1\linewidth]{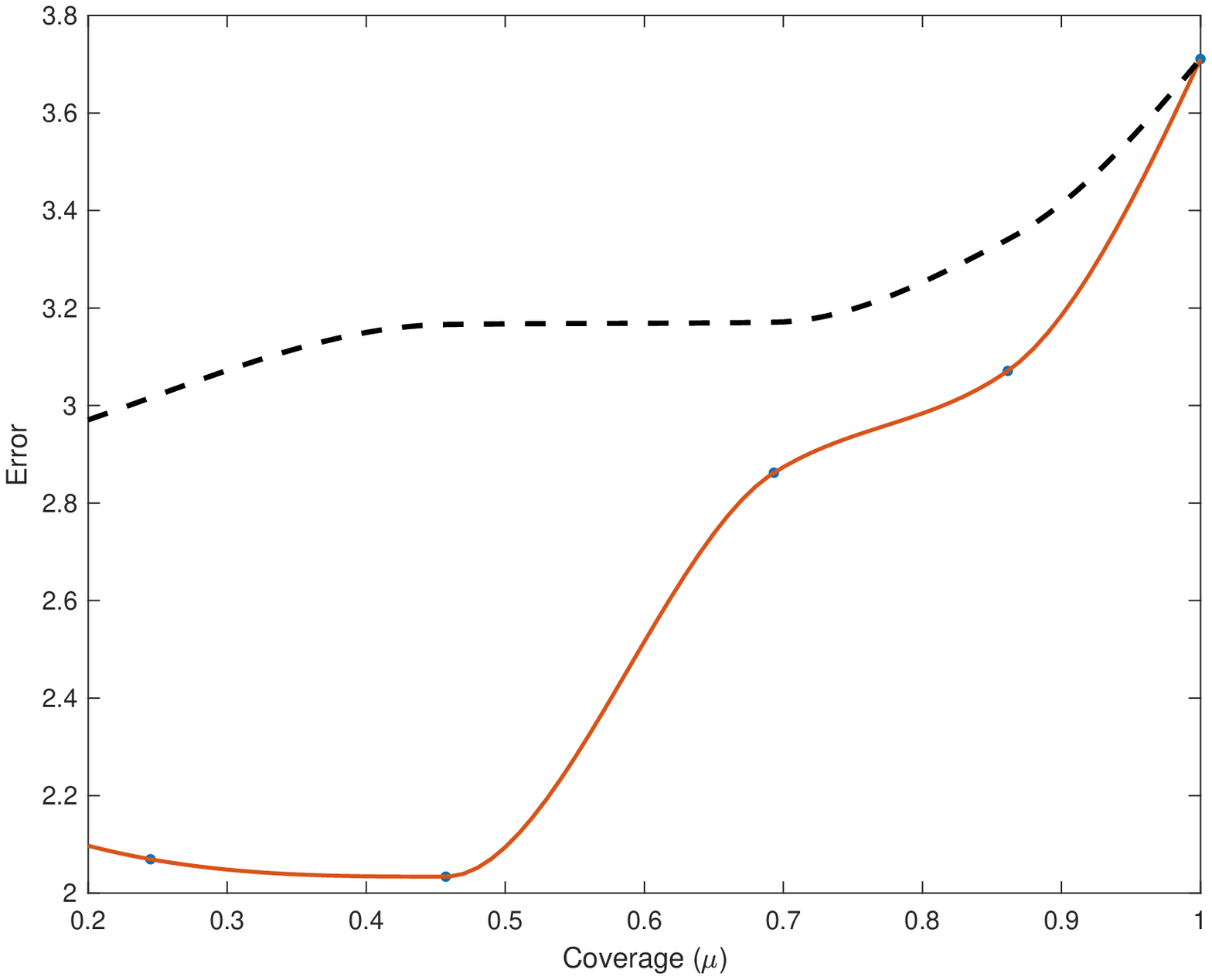}
  \caption{Boston housing}
  \label{fig:housing}
\end{subfigure}
\begin{subfigure}{.495\linewidth}
  \centering
  \includegraphics[width=1\linewidth]{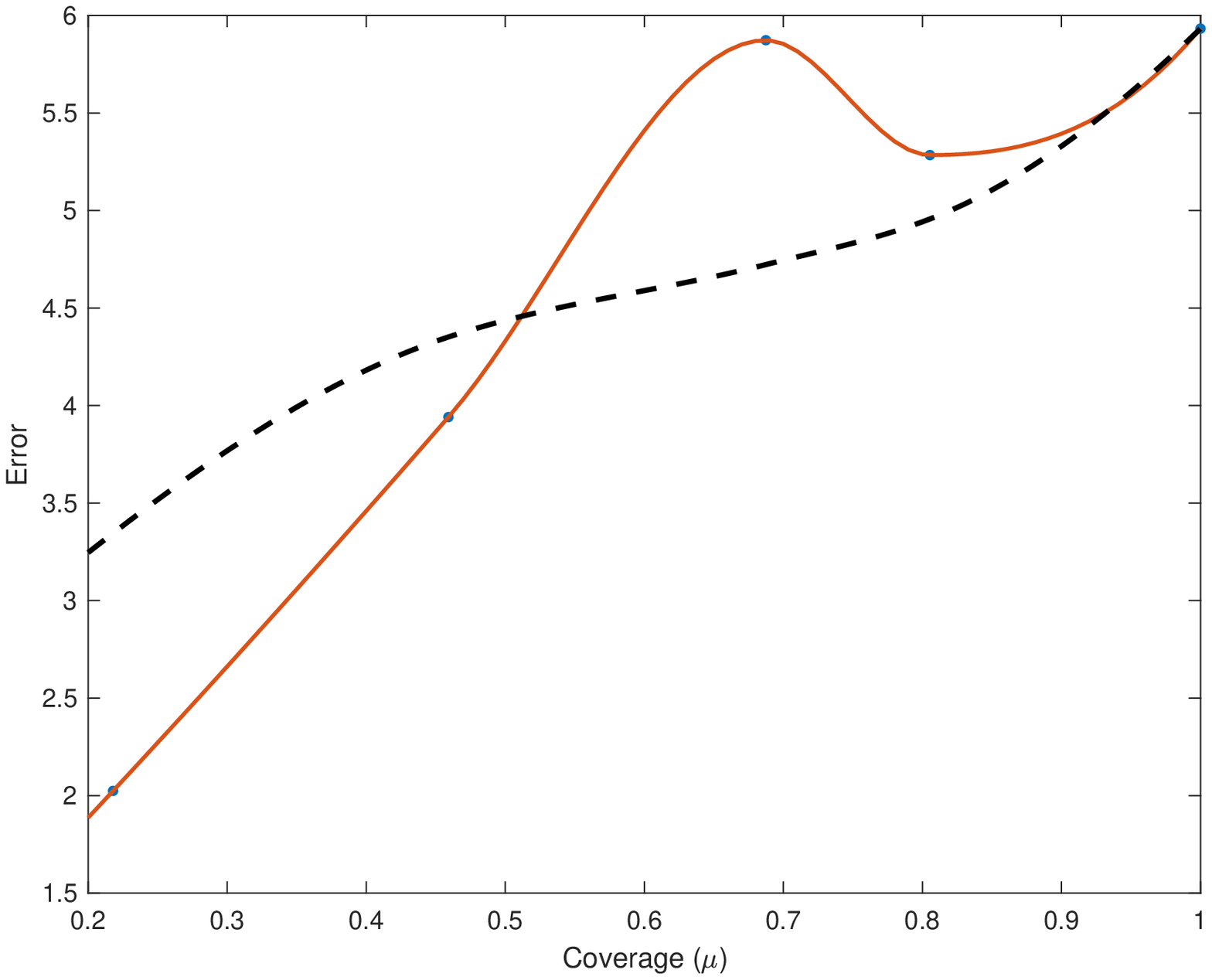}
  \caption{Cpusmall}
  \label{fig:cpus}
\end{subfigure}
\begin{subfigure}{.495\linewidth}
  \centering
  \includegraphics[width=1\linewidth]{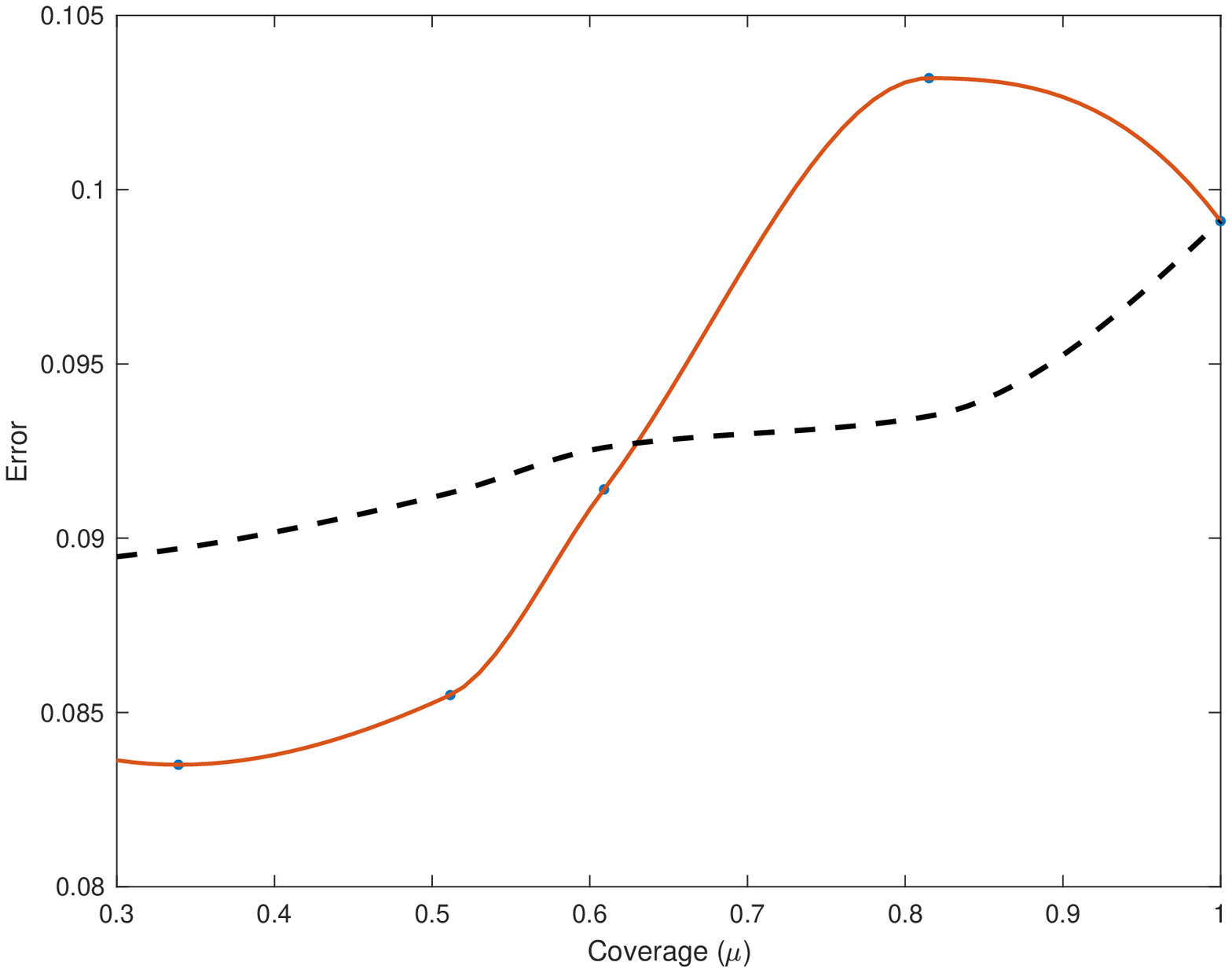}
  \caption{Space}
  \label{fig:space}
\end{subfigure}
\end{center}
\caption{RC curves of Algorithm~\ref{walg} (red line) and the baseline method (dashed line) on different regression data sets from the LIBSVM repository. The horizontal axis represents the coverage ($\mu$), while the vertical axis represents the error. }
\label{fig:RCCurves}
\end{figure}

\section{Directions for future work}
There are several natural open problems. First, although sparsity is 
desirable, our exponential dependence of the running time (or list size) 
on the sparsity is problematic. (The sup norm regression algorithm \citep{juba17} also suffered this deficiency.) Is it possible to avoid this? 
Second, our algorithm for reference class regression has an $O(n^k)$ blow-up of the loss, as compared to $\tilde{O}(n^{k/2})$ for conditional regression. Can we achieve a similar approximation factor for reference class regression?
Finally, we still do not know how close to optimal  this blow-up of the loss is; in particular, we do not have any lower bounds. 
Note that this is a computational and not a statistical issue, since we can obtain uniform convergence over all $k$-DNF conditions.

\bibliographystyle{plainnat}
\bibliography{robust}

\end{document}